\documentclass{article}
\usepackage{xcolor}         % colors
\usepackage{microtype}
\usepackage{graphicx}
\usepackage{natbib}
\usepackage{amssymb,amsmath, amsthm}
\usepackage{subfigure}
\usepackage{multirow}
\usepackage{rotating}
\usepackage{booktabs} % for professional tables
 
\usepackage{hyperref}
\usepackage{algorithm}
\usepackage[noend]{algcompatible}
\newtheorem{theorem}{Theorem}
\usepackage{fullpage}

\title{A Log-linear Gradient Descent Algorithm for Unbalanced Binary Classification using the All Pairs Squared Hinge Loss}
\author{Kyle Rust, Northern Arizona University
\and Toby Hocking, Northern Arizona University}

\begin{document}

% \author{%
%   \name Kyle Rust \email krr387@nau.edu  \\%jmlr:tdhock
%   \addr School of Informatics, Computing, and Cyber Systems\\
%   Northern Arizona University\\
%   Flagstaff, AZ 86001, USA
%   \AND
%   \name Toby Dylan Hocking \email toby.hocking@nau.edu \\%jmlr:tdhock
%   \addr School of Informatics, Computing, and Cyber Systems\\
%   Northern Arizona University\\
%   Flagstaff, AZ 86001, USA
%   }

% \editor{TBD}

% \textbf{Condensed running title of 50 characters or less:} Sort-Based Surrogate for AUC Optimization

% \textbf{Five keywords:} AUC, ROC, loss functions, gradient descent, optimization.

\maketitle

\begin{abstract}
Receiver Operating Characteristic (ROC) curves are plots of true positive rate versus false positive rate which are used to evaluate binary classification algorithms.
Because the Area Under the Curve (AUC) is a constant function of the predicted values, learning algorithms instead optimize convex relaxations which involve a sum over all pairs of labeled positive and negative examples.
Na\" ive learning algorithms compute the gradient in quadratic time, which is too slow for learning using large batch sizes.
We propose a new functional representation of the square loss and squared hinge loss, which results in algorithms that compute the gradient in either linear or log-linear time, and makes it possible to use gradient descent learning with large batch sizes.
In our empirical study of supervised binary classification problems, we show that our new algorithm can achieve higher test AUC values on imbalanced data sets than previous algorithms, and make use of larger batch sizes than were previously feasible.
\end{abstract}

\section{Introduction}
\label{sec:introduction}

Binary classification is an important problem in many areas such as computer vision, natural language processing, and bioinformatics.
Binary classification learning algorithms result in a function that outputs a real-valued predicted score (larger for more likely to be in the positive class).
The prediction accuracy of learned binary classification models can be quantified using the zero-one loss, which corresponds to thresholding the predicted score at zero. 
Because it only considers one prediction threshold (the default), this evaluation metric can be problematic and/or misleading in some cases (data sets with extreme class imbalance, models with different false positive rates). 
A more comprehensive and fair evaluation method involves the Receiver Operating Characteristic (ROC) Curve, which involves plotting True Positive Rate versus False Positive Rate, for all thresholds of the predicted score \citep{egan1975signal}.
The Area Under the ROC Curve (AUC) takes values between zero and one; constant/random/un-informed predictions yield AUC=0.5 and a set of perfect predictions would achieve AUC=1.
It is therefore desirable to create learning algorithms that maximize AUC, and that criterion is often used for hyper-parameter selection.
However, for gradient descent learning it is impossible to directly use the AUC since it is a piecewise constant function of the predicted values (the gradient is zero almost everywhere).
Various authors have proposed to work around this issue by using convex relaxations of the Mann-Whitney statistic \citep{bamber1975area}, which involves a double sum over all pairs of positive and negative examples. 
However, the utility of this method is limited to relatively small number of examples $n$, because of the quadratic $O(n^2)$ time complexity of the na\"ive implementation of the double sum.
This paper proposes a new, more efficient way to compute this double sum, for the special case of the square loss and squared hinge loss.
The main new idea is to use a functional representation of the loss, which results in a new algorithm which computes the all pairs square loss in linear $O(n)$ time, and a new algorithm which computes the all pairs squared hinge loss in log-linear $O(n\log n)$ time.

\subsection{Contributions and organization}

Our main contributions are new algorithms for computing the loss and gradient in sub-quadratic time, for loss functions based on square loss and squared hinge loss relaxations of the AUC.
In Section~\ref{sec:model} we describe the new functional representation of the loss.
In Section~\ref{sec:algorithms} we provide proofs that this functional representation can be used to compute the loss in sub-quadratic time.
In Section~\ref{sec:results} we provide an empirical study of supervised binary classification problems, and show that our method compares favorably with previous methods in terms of computation time and test AUC.
Section~\ref{sec:discussion} concludes with a discussion of the significance and novelty of our findings.

\subsection{Related work}
\label{sec:related-work}

There are several reviews about algorithms for dealing with class imbalance \citep{japkowicz2002class,batista2004study,kotsiantis2007supervised,he2009learning,krawczyk2016learning,haixiang2017learning,johnson2019survey}, among which a primary evaulation method is 
ROC curve analysis, which originates in the signal processing literature \citep{egan1975signal}. 
\paragraph{Loss functions which sum over examples, including re-weighting methods.}
One approach for AUC optimization involves the use of a standard loss function with example-specific weights which depend on the class imbalance \citep{ferri2002learning, cortes2004auc, cortes2007magnitude, Scott2012, wang2015optimizing}. 
\citet{cui2019class} proposed a re-weighting method based on the effective number of samples.
\citet{cao2019learning} proposed a margin-based loss which is aware of the label distribution.

\paragraph{Methods based on over-sampling and under-sampling.}
\citet{chawla2002smote} proposed a Synthetic Minority Over-sampling TEchnique (SMOTE), and there are related methods \citep{chawla2003smoteboost, han2005borderline,seiffert2009rusboost}.
Similarly, there are algorithms based on under-sampling \citep{liu2008exploratory}.
\cite{he2008adasyn} proposed an adaptive synthetic sampling algorithm (ADASYN).
More generally these techniques are referred to as data augmentation, for which there have been several surveys \citep{shorten2019survey}.

% \citet{provost1997analysis} brought a convex-hull method for ROC evaluation.
% \citet{ling2003auc} gave a proof of AUC consistency, while \citet{hand2009measuring} showed it is incoherent with respect to classification costs, and proposed the alternative ``H measure''.
% Losses can often be interpreted as a convex relaxation of the zero-one loss summed over all examples. 
% Other algorithms to estimate the global AUC have been proposed \citep{ rakotomamonjy2004optimizing, herschtal2004optimising, herschtal2006area,  t}. 
% \citet{Han2010} posits a linear model and algorithm that maximizes the AUC utilizing active learning. \citet{zhao2011online} created an AUC-maximizing online learning algorithm. \citet{Menon2013} examined class imbalance and statistical consistency using a mean of true positive and true negative rates.
% used that approxmation to propose the AUCtron linear-model algorithm.  \citet{ying2016stochastic} suggested solving a saddle point problem involving a pairwise square loss for a stochastic AUC optimization, and  \citet{yuan2020auc} proposed extending this to the squared hinge loss.
\renewcommand{\arraystretch}{1.25}
\begin{table}[ht]
\begin{center}
\begin{tabular}{ |c|c|c|c|c|c|} 
 \hline

 Paper & Degree & Hinge & Proof & Solution\\
 \hline
 \citet{pahikkala_tsivtsivadze_airola_jarvinen_boberg_2009} & Square  & False & False & Functional \\ %TODO Verfiy no proof
 \hline
 \citet{joachims2005multivariate} & Linear & True & True & Functional \\ 
 \hline
 \citet{calders2007efficient} & Polynomial & False & True & Functional \\ 
 \hline
 \citet{ying2016stochastic} & Square & False & True &Min-Max \\
 \hline
 \citet{yuan2020auc} & Square & True & True & Min-Max \\
 \hline
 This Work & Square & True & True & Functional\\
 \hline
\end{tabular}
\end{center}
\caption{A table outlining the novel contributions of work relating to the sub-quadratic computation of the linear, square, and squared hinge loss functions. The table outlines the degree of the loss function, whether a hinge is included in the loss function, if the associated paper includes a proof, and the context how the loss is represented.}
\label{table:related-work} 
\end{table}

\paragraph{Loss functions which sum over all pairs of positive and negative examples.}
\citet{bamber1975area} showed the equivalence of the ROC-AUC and the Mann-Whitney test statistic \citet{mann1947test}, which has led to many proposed algorithms based on objective functions that are convex surrogates of the Mann-Whitney statistic \citep{yan2003optimizing, castro2008optimization, ying2016stochastic, yuan2020auc, calders2007efficient, tino2022rankings-problems-review}. 
\citet{joachims2005multivariate} proposed a support vector machine algorithm (in quadratic time) to maximize AUC based on a pairwise loss function, and \citet{freund2003efficient} proposed an approach based on boosting. \citet{joachims2006linearSVM} proposed another support vector machine algorithm that could be trained in linear time; this work was generalized by \citet{AIROLA20111328} to make use of real-valued utility scores. 
\citet{lee2014ranksvm} extends this work by investigating methods that makes use of trees to compute pair-wise values. \citet{narasimhan2013structural} extended this to the partial AUC, or the area under the curve between two false positive rates (not necessarily 0 and 1). 
\citet{kotlowski2011bipartite} covered the difference between risk and regret in losses which sum over examples versus pairs. 
\citet{rudin2005margin} demonstrated losses which sum over examples can yield the same solution as a pairwise loss. 
\citet{pahikkala2007rank-pairwise-regularized-least-squares} proposed using an algorithm based on regularized least squares to minimize the number on incorrectly ranked pairs of data points. 
\citet{pahikkala_tsivtsivadze_airola_jarvinen_boberg_2009} followed up by proposing a learning algorithm that can be trained using pairwise class probabilities. 
A comprehensive review of ranking problems can be found in \citet{tino2022rankings-problems-review}.

\paragraph{Novelty with respect to previous work.} 
In this paper we propose a new algorithm for computing the square loss with respect to all pairs of positive and negative examples in log-linear time.
With respect to previous work, the novelty of our proposed algorithm can be understood in terms of four components (columns in Table~\ref{table:related-work}). 
These components include the polynomial degree of the loss, the inclusion of the hinge in the proposed loss function, whether or not a proof is presented in the previous paper, and what problem space the solution falls into. 
There are five previous papers which are most closely related to our proposed algorithm (rows in Table~\ref{table:related-work}). 
The \citet{joachims2005multivariate} paper proposes a faster computation method for the L1/linear degree loss function, while the other papers rely on L2/square or greater degree to create more disparity between correct and incorrect predictions.
The approach taken by the \citet{calders2007efficient} group is to compute a polynomial approximation of the square loss that can be computed efficiently. 
The \citet{pahikkala_tsivtsivadze_airola_jarvinen_boberg_2009}, \citet{ying2016stochastic}, and the \citet{calders2007efficient} works contain loss functions that do not make use of a hinge at the margin parameter.
The \citet{pahikkala_tsivtsivadze_airola_jarvinen_boberg_2009} paper also does not provide a proof for their proposed function. 
Finally, the LIBAUC group turn the computation of the square and squared hinge loss into a Min-Max problem  \citep{ying2016stochastic, yuan2020auc}. 
In summary, our paper is novel because it is the first to provide rigorous mathematical proof that the all pairs squared hinge loss can be computed in log-linear time, which is much faster than the quadratic time used by a na\" ive implementation.

\section{Models and Definitions}
\label{sec:model}
In this section we give a formal definition of the all pairs squared hinge loss.
For notation in this section we use $i$ for indices from 1 to $n$ (both positive and negative examples), whereas we use $j$ for only positive and $k$ for only negative examples.
\subsection{Supervised binary classification}
In supervised binary classification we are given a set of $n$ labeled training examples, $\{(\mathbf x_i, y_i)\}_{i=1}^n$ where $\mathbf x_i\in\mathbb R^p$ is an input feature vector and $y_i\in\{-1,1\}$ is a binary output/label.
Let $\mathcal I^{+}=\{j:y_j = 1\}$ be the set of indices of positive examples, let $n^+ = |\mathcal I^+|$ be the number of positive examples, and let $\mathcal I^{-}, n^-$ be the analogous set/number of negative examples.

\paragraph{Loss functions which sum over examples.} 
In typical balanced binary classification problems, we want to learn a function $f:\mathbb R^p\rightarrow \mathbb R$ that computes real-valued predictions $\hat y_i=f(\mathbf x_i)$ with the same sign as the corresponding label $y_i$.
Typical learning algorithms involve gradient descent using a convex surrogate loss function $\ell:\mathbb R\rightarrow \mathbb R_+$, which is summed over all training examples:
\begin{equation}
\label{eq:loss-sum-over-examples}
    \mathcal L(f) =  \sum_{i=1}^n \ell[ y_i f(\mathbf x_i) ].
\end{equation}
Large $y_i f(\mathbf x_i)>0$ values result in correctly predicted labels, whereas small $y_i f(\mathbf x_i)<0$ values cause incorrectly predicted labels.

\paragraph{Pairwise loss functions.} 
To more fairly compare different algorithms for binary classification, the Area Under the ROC Curve (AUC) is often used as an evaluation metric to maximize. 
Maximizing the AUC is equivalent to minimizing the Mann-Whitney test statistic \citet{bamber1975area}, which is defined in terms of a sum of indicator functions over all pairs of positive and negative examples.
Corresponding learning algorithms involve summing a convex surrogate $\ell$ over all pairs of positive and negative examples,
\begin{equation}
\label{eq:loss-sum-over-pairs}
    \mathcal L(f) =  
    \sum_{j\in\mathcal I^+}
    \sum_{k\in\mathcal I^{-}}
    \ell[ f(\mathbf x_j) - f(\mathbf x_k) ].
\end{equation}
Large pairwise difference values $f(\mathbf x_j) - f(\mathbf x_k) > 0$ result in correctly ranked pairs, whereas small pairwise difference values $f(\mathbf x_j) - f(\mathbf x_k) < 0$ cause incorrectly ranked pairs.

\paragraph{Choice of convex relaxation $\ell$.}
Both kinds of loss functions (summed over examples or pairs) depend on the choice of the $\ell:\mathbb R\rightarrow \mathbb R_+$ function.
Using the zero-one loss, $\ell(z)=I[z < 0]$, where $I$ is the indicator function (1 if true, 0 if false), means that (\ref{eq:loss-sum-over-examples}) counts incorrectly classified examples, whereas (\ref{eq:loss-sum-over-pairs}) counts pairs of positive and negative examples which are ranked incorrectly (positive example has a predicted value less than the negative example).
A typical choice for a convex relaxation of the zero-one loss is the logistic loss $\ell(z)=\log[1+\exp(-z)]$, which can be derived from the objective of maximizing the binomial likelihood.
Other convex relaxations are motivated by geometric rather than probabilistic arguments, and depend on a
margin size hyper-parameter $m\geq 0$ (which can be chosen via cross-validation, or kept at a default of $m=1$). 
A popular geometric convex surrogate used by the support vector machine is the (linear) hinge loss $\ell(z)=(m-z)_+$, where $(z)_+ = z I[z > 0]$ is the positive part function.
Other typical choices for geometric convex surrogates include the square loss $\ell(z)=(m-z)^2$ and the squared hinge loss $\ell(z)=(m-z)^2_+$, which we study in this paper in the pairwise context.

\subsection{Functional loss}
We propose to compute the double sum over pairs  (\ref{eq:loss-sum-over-pairs}) using a single sum over functions.
The function that computes the loss over all positive examples is:
\begin{equation}
\label{eq:L+}
    \mathcal L^+(x) = \sum_{j\in\mathcal I^+}
    \ell( \hat y_j - x ).
\end{equation}
The total pairwise loss (\ref{eq:loss-sum-over-pairs}) can be written using a sum over evaluations of this function,
\begin{equation}
\label{eq:loss_sum_pos_or_neg}
    \mathcal L(f) =  \sum_{k\in\mathcal I^{-}} \mathcal L^+(\hat y_k).
\end{equation}
Na\" ively using (\ref{eq:L+}) to compute each $\mathcal L^+(\hat y_k)$ value requires $O(n^+)$ time, which results in an overall time complexity of $O(n^+ n^-)$ for computing the loss value $\mathcal L(f)$.
Assuming that the number of positive/negative examples scales with the data set size, $n^+=O(n)$ and $n^-=O(n)$, then the overall time complexity of computing the loss value is quadratic $O(n^2)$, which is much too slow when learning with large data sets.
The main novelty of our paper is deriving alternative functional representations of $\mathcal L^+$ which require only constant $O(1)$ time to compute each loss value $\mathcal L^+(\hat y_k)$.
This results in an algorithm for computing the square loss which is linear $O(n)$, and an algorithm for computing the squared hinge loss which is log-linear $O(n\log n)$, as we explain in the next section.

\section{Algorithms}
\label{sec:algorithms}
In this section, we begin by proving that the all pairs square loss can be computed in linear time, and then prove that the all pairs squared hinge loss can be computed in log-linear time (whereas both are quadratic time using na\" ive implementations).
The main idea of our approach is to compute an exact representation of the loss function $\mathcal L^+$ which can be efficiently evaluated for each negative predicted value $\hat y_k$.
Because we are considering the special case of $\ell$ being the square loss or the squared hinge loss, we can use coefficients $a,b,c\in\mathbb R$ to represent a function,
\begin{equation}
    G_{a,b,c}(x) = ax^2 + bx + c.
\end{equation}
Then for every positive example $j\in\mathcal I^+$, we have a prediction $\hat y_j$, and if we pair this example with a negative example with predicted value $x$, then the loss of that pair can be computed using the function
\begin{equation}
    h_j(x) =  (x + m-\hat y_j)^2 = x^2 + 2(m- \hat y_j)x + (m-\hat y_j)^2 = G_{1,2(m-\hat y_j),(m-\hat y_j)^2}(x),
\end{equation}
where $m$ is the margin hyper-parameter.

The main insight of our paper is that this functional representation yields an efficient algorithm for computing the total loss over all positive examples,
\begin{equation}
\label{eq:L+k_functional}
  \mathcal L^+(x) = 
  \sum_{j\in\mathcal J(x,m)} h_j(x) = 
  \sum_{j\in\mathcal J(x,m)} G_{1,2(m-\hat y_j),(m-\hat y_j)^2}(x) =
  G_{A(x), B(x), C(x)}(x).
\end{equation}
where $\mathcal J(x,m)$ is the set of positive indices with non-zero loss, and $A(x), B(x), C(x)$ are the corresponding coefficients.
In the case of the square loss, $\mathcal J_{\text{square}}(x,m)=\mathcal I^+$ is the complete set of indices.
In the case of the squared hinge loss, $\mathcal J_{\text{squared hinge}}(x,m)=j:\hat y_j - x < m$ is the set of indices with non-zero loss.
The coefficients are defined by summing over all indices in the set,
\begin{eqnarray}
A(x) &=&  \sum_{j\in\mathcal J(x,m)} 1 = |\mathcal J(x,m)|, \label{eq:ax} \\
B(x) &=& \sum_{j\in\mathcal J(x,m)} 2(m-\hat y_j), \\
C(x) &=& \sum_{j\in\mathcal J(x,m)} (m-\hat y_j)^2.\label{eq:cx} 
\end{eqnarray}

Importantly, using ($\ref{eq:L+k_functional}$) to compute each $\mathcal L^+(\hat y_k)$ loss value requires only constant $O(1)$ time, given the coefficients $A(x), B(x), C(x)$.
As will be seen, 
The algorithm computes the total loss using $(\ref{eq:loss_sum_pos_or_neg})$, which requires a sum over all $n^-$ negative examples, so the best case time complexity would be $O(n^-)$.
However computing the coefficients requires either log-linear $O(n\log n)$ time for the squared hinge loss, or linear $O(n^+)$ time for the square loss, as shown below.

\subsection{Square loss (no hinge)}

The all pairs square loss has been proposed as a surrogate because it is easy to analyze theoretically \citet{ying2016stochastic,liu2019stochastic}.
This loss function entails using $\ell(z) = (m-z)^2$ in (\ref{eq:loss-sum-over-pairs}).

Our proposed algorithm for computing the square loss begins by computing three coefficients,
\begin{eqnarray}
a^+ &=&  \sum_{j\in\mathcal I^+} \underbrace{1}_{a_j} = n^+ = |\mathcal I^+|, \label{eq:a+} \\
b^+ &=& \sum_{j\in\mathcal I^+} \underbrace{2(m-\hat y_j)}_{b_j}, \\
c^+ &=& \sum_{j\in\mathcal I^+} \underbrace{(m-\hat y_j)^2}_{c_j}.\label{eq:c+} 
\end{eqnarray}
These three coefficients are an exact representation of the loss function $\mathcal L^+$,
\begin{equation}
    \mathcal L^+(x) = a^+ x^2 + b^+ x + c^+.
\end{equation}
The algorithm then evaluates this function for every negative prediction, and takes the sum of all the loss values,
\begin{equation}
\label{eq:func-square}
    \sum_{k\in\mathcal I^-} \mathcal L^+(\hat y_k) = \sum_{k\in\mathcal I^-} a^+ \hat y_k^2 + b^+ \hat y_k + c^+.
\end{equation}
The following theorem states that this functional representation can be used to efficiently compute the pairwise square loss:

\begin{theorem}
If $\ell$ is the square loss then the total loss over all pairs of positive and negative examples (\ref{eq:loss-sum-over-pairs}) can be computed via the functional loss (\ref{eq:func-square}), that is
\end{theorem}
\begin{equation}
    \sum_{k\in\mathcal I^{-}}
    \sum_{j\in\mathcal I^+}
    \ell( \hat y_j  - \hat y_k )
    =
\sum_{k\in\mathcal I^-} a^+ \hat y_k^2 + b^+ \hat y_k + c^+.
\end{equation}
\begin{proof}
It suffices to show that for any $k\in\mathcal I^-$ we have
\begin{eqnarray}
\sum_{j\in\mathcal I^+}
\ell( \hat y_j  - \hat y_k )
&=&
\label{eq:use-def-square-loss}
\sum_{j\in\mathcal I^+}
(m-\hat y_j  + \hat y_k)^2\\
&=&
\label{eq:use-expand-square}
\sum_{j\in\mathcal I^+}
(m-\hat y_j)^2 + 2(m-\hat y_j)\hat y_k + \hat y_k^2\\
&=& 
\label{eq:use-def-abc}
a^+ \hat y_k^2 + b^+ \hat y_k + c^+.
\end{eqnarray}
The proof above uses the definition of $\ell$ as the square loss with a margin of $m$ (\ref{eq:use-def-square-loss}), expands the square (\ref{eq:use-expand-square}), then uses the definitions of $a^+,b^+,c^+$ to complete the proof (\ref{eq:use-def-abc}).
\end{proof}

Figure~\ref{fig:geometric-square} demonstrates a geometric interpretation of how the proposed all pairs square loss is computed. The left panel displays how for every predicted value that has a label that that exists in $\mathcal{I^+}$, the coefficients $a_j, b_j, c_j$ are computed (\ref{eq:a+}--\ref{eq:c+}). 
The right panel shows that the precomputed coefficients can be summed, then for each predicted value that has a label in $\mathcal{I^-}$ the resulting curve can be evaluated and added to the total loss.

\begin{figure}[H]
    \centering
    \includegraphics[width = 0.9\textwidth]{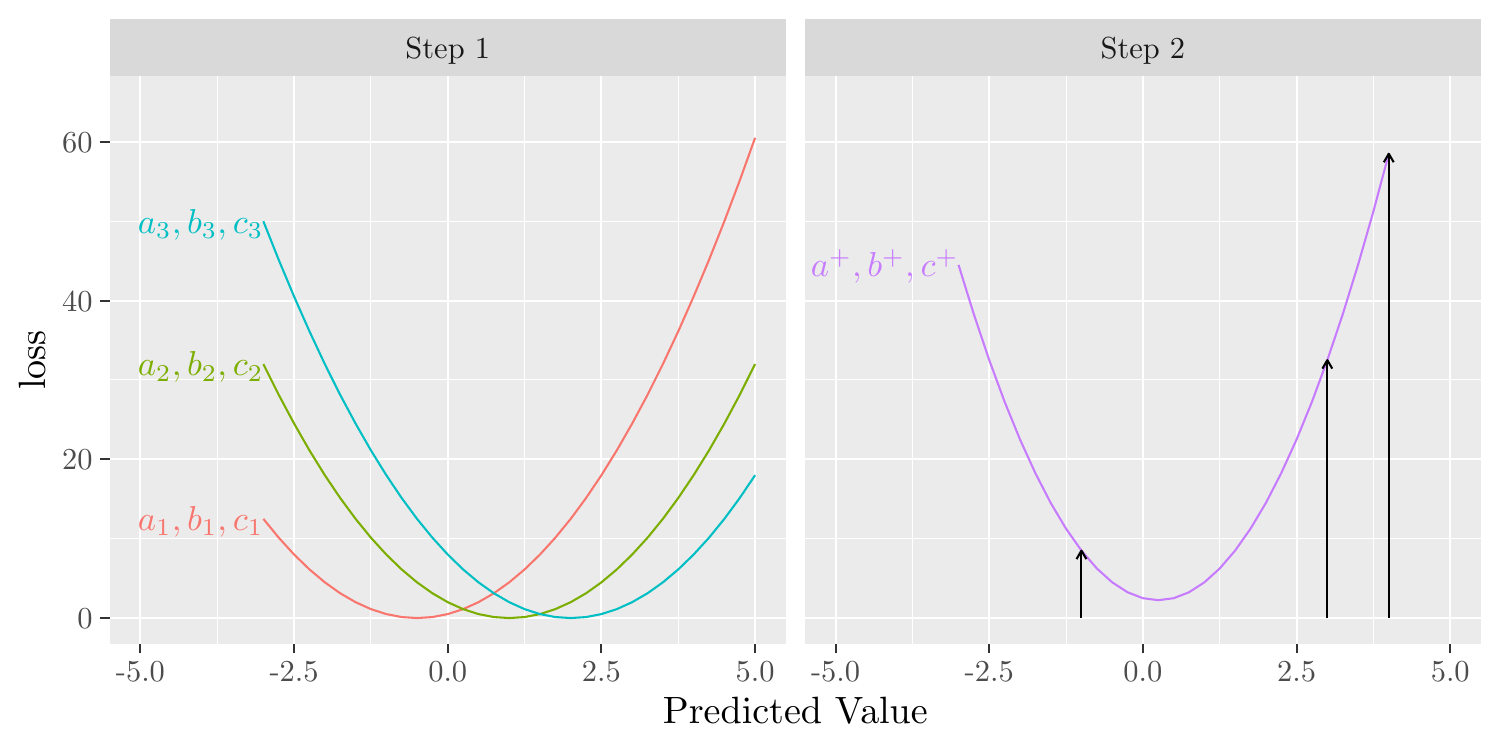}
    \caption{Geometric interpretation of how our proposed algorithm computes the all pairs square loss. \textbf{Left, Step 1:} for every predicted value associated with a positive label $j$, compute the coefficients $a_j, b_j, c_j$. 
    The red, green, and blue curves represent the square loss functions for three different observations $j$ with positive labels. 
    \textbf{Right, Step 2:} summing these coefficients over all observations $j$ with positive labels results in the purple curve, which represents the total square loss, that is evaluated for every positive label (black arrows).}
    \label{fig:geometric-square}
\end{figure}

\begin{algorithm}[H]
\caption{\label{alg:all-pairs-loss-square}All Pairs Square Loss (No hinge)}
\begin{algorithmic}[1]
\STATE Input: Predictions $\hat y_1,\dots, \hat y_n\in\mathbb R$, indices $\mathcal I^+,\mathcal I^-$, margin size $m\geq 0$.\label{line:inputs-square}
\STATE Initialize to zero the coefficients $a,b,c\in\mathbb R$ and loss $ L\in\mathbb R$.\label{line:init-coefs-loss-square}
\FOR{$j$ in $\mathcal I^+$}:\label{line:for-i-1-n-square}
    \STATE $z\gets m-\hat y_{j}$, $a\text { += } 1$, $b\text { += } 2z$, $c\text { += } z^2$ \label{line:coefs-square}
\ENDFOR
\FOR{$k$ in $\mathcal I^-$}:\label{line:for-i-1-n-square-second}
    \STATE $ L \text{ += } a\hat y_{k}^2 + b\hat y_{k} + c$.\label{line:loss-square}
\ENDFOR
\STATE Output: total loss $ L$.\label{line:output-square}
\end{algorithmic}
\end{algorithm}
Algorithm~\ref{alg:all-pairs-loss-square} summarizes the proposed method for computing the square loss over all pairs of positive and negative examples.
% It assumes input of index sets $\mathcal I^+,\mathcal I^-$ (line~\ref{line:inputs-square}), which can be computed in linear $O(n)$ time given the labels $y_1,\dots,y_n$.
% The coefficients and loss are initialized to zero (line~\ref{line:init-coefs-loss-square}), then the for loop over positive examples (line~\ref{line:for-i-1-n-square}) is used to update the coefficients (line~\ref{line:coefs-square}).
% The for loop over negative examples (line~\ref{line:for-i-1-n-square-second}) is then used to update the loss (line~\ref{line:loss-square}), which is the output of the algorithm (line~\ref{line:output-square}).
% Overall the time complexity of Algorithm~\ref{alg:all-pairs-loss-square} is linear, $O(n)$.

\subsection{Squared hinge loss}
The squared hinge loss is a preferable convex surrogate for AUC optimization because of its consistency properties \citep{yuan2020auc}.
This loss function involves using $\ell(z) = (m-z)^2_+$ in (\ref{eq:loss-sum-over-pairs}).
In this case, the algorithm begins by computing a predicted value for every $i\in\{1,\dots,n\}$ which is augmented only for negative examples by the amount of the margin hyper-parameter $m\geq 0$,
\begin{equation}
    v_i = \hat y_i + m I[y_i = -1].
\end{equation}
We then sort these augmented prediction values; let the indices $\{s_1,\dots, s_n\}$ be a permutation of $\{1,\dots,n\}$ such that $v_{s_1}\leq \cdots \leq v_{s_n}$.
These values can be used to write the loss function to evaluate for any negative example $k$ in the sequence,
\begin{equation}
    H_k(x) = \sum_{j=1}^k h_{s_j}(x) I[y_{s_j} = 1] = G_{a_k,b_k,c_k}(x). \label{eq:Hi}
\end{equation}
Note that the $H_k$ function can be efficiently represented in terms of the three real-valued coefficients $a_k,b_k,c_k$.
Our proposed algorithm starts with coefficients $a_0,b_0,c_0=0$ and loss $L_0=0$, then computes the following recursive updates for all $i\in\{1,\dots,n\}$,
\begin{eqnarray}
a_i &=& a_{i-1} + I[y_{s_i}=1], \label{eq:a} \\
b_i &=& b_{i-1} + I[y_{s_i}=1] 2(m-\hat y_{s_i}), \\
c_i &=& c_{i-1} + I[y_{s_i}=1] (m-\hat y_{s_i})^2,\label{eq:c} \\
L_i &=& L_{i-1} + I[y_{s_i}=-1] ( a_i \hat y^2_{s_i} + b_i \hat y_{s_i} + c_i).\label{eq:recursive-cost}
\end{eqnarray}

\begin{algorithm}[t]
\caption{\label{alg:all-pairs-loss-l2}Log-linear time computation of All Pairs Squared Hinge Loss}
\begin{algorithmic}[1]
\STATE Input: Predictions $\hat y_1,\dots, \hat y_n\in\mathbb R$, labels $ y_1,\dots,  y_n\in\{-1,1\}$, margin size $m\geq 0$.\label{line:inputs-alg-1}
\STATE Initialize to zero the coefficients $a,b,c\in\mathbb R$ and loss $ L\in\mathbb R$.\label{line:init-coefs-loss}
\STATE $v_i\gets \hat y_i + m I[y_i = -1]$ for all $i\in\{1,\dots,n\}$.\label{line:augment-predictions-alg-1}
\STATE $s_1,\dots,s_n\gets\textsc{SortedIndices}(v_1,\dots,v_n)$.\label{line:sort-alg-1}
\FOR{$i$ from 1 to $n$}:\label{line:for-i-1-n}
  \IF{$y_{s_i} = 1$}:
    \STATE $z\gets m-\hat y_{s_i}$, $a\text { += } 1$, $b\text { += } 2z$, $c\text { += } z^2$ \label{line:coefs}
  \ELSE
    \STATE $ L \text{ += } a\hat y_{s_i}^2 + b\hat y_{s_i} + c$.\label{line:loss}
  \ENDIF
\ENDFOR
\STATE Output: total loss $ L$ (gradients of $\hat y_1,\dots, \hat y_n$ can be computed using automatic differentiation)\label{line:output}
\end{algorithmic}
\end{algorithm}

The following theorem states that these update rules can be used to efficiently compute the pairwise squared hinge loss:
\begin{theorem}
If $\ell$ is the squared hinge loss then the total loss over all pairs of positive and negative examples (\ref{eq:loss-sum-over-pairs}) can be computed via the recursion  (\ref{eq:recursive-cost}), that is $
    \sum_{k\in\mathcal I^{-}}
    \sum_{j\in\mathcal I^+}
    \ell( \hat y_j  - \hat y_k ) 
    = 
    L_n.$
\end{theorem}
\begin{proof}
We first re-write the cost in terms of the sorted indices, using the definition of $H_i$~(\ref{eq:Hi}),
\begin{equation}
    \sum_{k\in\mathcal I^{-}}
    \sum_{j\in\mathcal I^+}
    \ell( \hat y_j  - \hat y_k )
    =
    \sum_{k=1}^n
    I[y_k = -1]
    \sum_{j=1}^n
    I[y_j = 1]
    \ell( \hat y_j  - \hat y_k )
    =
    %\sum_{k\in\mathcal I^-} \mathcal H_k^+(\hat y_k) =
    \sum_{k=1}^n 
    I[y_{s_k} = -1]
    H_k(\hat y_{s_k}) .
\end{equation}
The proof is by induction. 
\paragraph{Base Case.} For the base case, when $y_{s_1}=1$ we have $I[y_{s_1}=-1]=0$ in~(\ref{eq:recursive-cost}, so $L_1 = 0$.
When $y_{s_1}=-1$ we have $a_1,b_1,c_1=0$ in~\ref{eq:recursive-cost}, which implies $L_1=0.$ 
The initial pairwise loss is indeed zero because there are no pairs to sum over (there is only one example, either positive or negative), so that proves the base case.

\paragraph{Induction.} 
We now assume that $L_t =\sum_{k=1}^t I[y_{s_k} = -1] H_k(\hat y_{s_k}) $ is true for all $t<n$ (induction hypothesis). 
If $y_{s_n}=1$ we have $I[y_{s_n}=-1]=0$ in~(\ref{eq:recursive-cost}), so
\begin{eqnarray}
L_n &=& L_{n-1}, \label{eq:use-recursive-cost} \\
 &=& \sum_{k=1}^{n-1} I[y_{s_{k}} = -1] H_k(\hat y_{s_{k}}), \label{eq:induction-assumption} \\
 &=& \sum_{k=1}^{n} I[y_{s_{k}} = -1] H_k(\hat y_{s_{k}}). \label{eq:label-is-positive}
\end{eqnarray}
Equation~(\ref{eq:use-recursive-cost}) follows from the last label being positive in the definition of the recursive loss~(\ref{eq:recursive-cost}).
Equation~(\ref{eq:induction-assumption}) follows from the induction hypothesis.
Equation~(\ref{eq:label-is-positive}) changes the sum up to $n-1$ to a sum up to $n$, which follows from the last label being positive, which proves that case.

The other case is when the last label is negative, $y_{s_n}=-1$, for which we have
\begin{eqnarray}
H_n(\hat y_{s_n}) 
&=& \sum_{j=1}^n\label{eq:squared-hinge}
\hat y_{s_n}^2 + 2(m-\hat y_j)\hat y_{s_n} + (m-\hat y_j)^2, \\
&=& \sum_{j=1}^{n-1} I[y_i = 1] \left[
\hat y_{s_n}^2  + 2(m-\hat y_{s_j}) \hat y_{s_n} + (m-\hat y_{s_j})^2 \label{eq:sum-sorted-indices}
\right], \\
&=& \label{eq:use-abc} a_{n-1} \hat y_{s_n}^2 + b_{n-1} \hat y_{s_n} + c_{n-1},\\
&=& a_{n} \hat y_{s_n}^2 + b_{n} \hat y_{s_n} + c_{n}.\label{eq:last-abc-zero}
\end{eqnarray}
The equality~(\ref{eq:squared-hinge}) follows from the definition of $\mathcal L$ and by only summing over terms $j$ such that the squared hinge loss is non-zero.
The equality~(\ref{eq:sum-sorted-indices}) follows from the definitions of $\mathcal I^+$ and the sorted indices $s_1,\dots,s_n$.
Finally, (\ref{eq:use-abc}) follows from the definitions of $a_i,b_i,c_i$ in (\ref{eq:a}--\ref{eq:c}), and 
(\ref{eq:last-abc-zero}) follows from the assumption that $y_{s_n}=-1$.
Using the induction hypothesis finishes the proof.
\end{proof}

Algorithm~\ref{alg:all-pairs-loss-l2} summarizes the proposed method for computing the squared hinge loss over all pairs of positive and negative examples.
It inputs a vector of predictions $\hat y_1,\dots,\hat y_n\in\mathbb R$, corresponding labels $y_1,\dots,y_n\in\{-1,1\}$, as well as a margin size parameter $m\geq 0$ (line~\ref{line:inputs-alg-1}).
The algorithm begins by initializing the loss as well as the three coefficients to zero (line~\ref{line:init-coefs-loss}), then computing augmented predictions $v_i$ (line~\ref{line:augment-predictions-alg-1}), and sorting them (line~\ref{line:sort-alg-1}). 
In each iteration of the for loop from smallest to largest augmented predicted values (line~\ref{line:for-i-1-n}), either the three coefficients are updated if the label is positive (line~\ref{line:coefs}), or they are used to update the loss if the label is negative (line~\ref{line:loss}).
Upon completion the algorithm the outputs the total loss value $L$ (line~\ref{line:output}).
Note that working in the ``forward'' direction (from smallest to largest augmented predicted value) means that Algorithm~\ref{alg:all-pairs-loss-l2} sums the $\mathcal L^+$ function over all negative examples; we could equivalently compute the loss by working  ``backward'' (from largest to smallest augmented predicted value) via summing the $\mathcal L^-$ function over all positive examples (the proof and pseudo-code are analogous).
Overall the space complexity of Algorithm~\ref{alg:all-pairs-loss-square} is $O(n)$ and the time complexity is $O(n\log n)$ due to the sort (line~\ref{line:sort-alg-1}).

% If automatic differentiation is Not available, then the gradient can be computed via Algorithm~\ref{alg:all-pairs-grad-l2}. The first for loop is over the two directions, forward and backward (line~\ref{line:for-direction}).
% For each direction the two coefficients are first initialized to zero (line~\ref{line:init-zero-coefs}), then there is a loop over the sorted data (line~\ref{line:for-first-last}).
% Each iteration either updates the two coefficients (line~\ref{line:coefs-grad}), or uses them to compute an element of the gradient vector (line~\ref{line:grad}).
% The output is the gradient vector, with $n$ elements (line~\ref{line:output-grad}).
% Algorithm~\ref{alg:all-pairs-grad-l2} requires $O(n)$ space and $O(n\log n)$ time.

\section{Empirical Results}
\label{sec:results}

\subsection{Proposed algorithm results in order of magnitude speedups}

To explore how the speed of our method compares to the logistic loss and the na\"ive approach (brute force sum over all pairs), we performed the following experiment which compares the computation times of the different methods.
We expected that there should be large empirical speedups for our proposed algorithm, consistent with its theoretical log-linear time complexity (much faster than the quadratic time na\"ive approach).
For each data size $n\in\{10^1,\dots,10^7\}$ we simulated $n$ standard normal random numbers to use as predictions $\hat y_1,\dots,\hat y_n$, and used an equal number of positive and negative labels.
We then measured the time to compute each loss value and gradient vector,
using a laptop with a 2.40GHz Intel(R) Core(TM)2 Duo CPU P8600.
We report the computation times as a function of data size $n$ in Figure~\ref{fig:timing-grad-square}.
``Na\"ive'' denotes the brute force method of summing over all pairs of positive and negative examples, while ``Functional" refers to the algorithms proposed in section \ref{sec:algorithms}. 
We observed on the log-log plot that the na\"ive approaches have a much larger asymptotic slope than the Functional methods, which is consistent with the theoretical expectation (quadratic for na\"ive, log-linear for Functional).
In particular we observed that in 1 second, the na\"ive approach is capable of computing the loss value and gradient vector for data on the order of $n\approx 10^3$ examples, whereas the Functional approach allows for greater sizes of up to $n\approx 10^6$ in the same amount of time.
In our timings, we observed that the na\"ive method is substantially slower for $\approx 1000$ examples or more.
In addition we observed that the Functional methods have a similar asymptotic slope as the logistic loss, which indicates that our proposed method is extremely fast (log-linear), almost as fast as the linear time logistic loss.
Of course the exact timings we observed in this experiment are dependent on the computer hardware that was used, but overall our analysis suggests our proposed algorithm results in substantial asymptotic speedups over the na\"ive method, and therefore allows for learning using much larger batch sizes than were previously practical.

\begin{figure}
    \centering
    \includegraphics[width=0.8\textwidth]{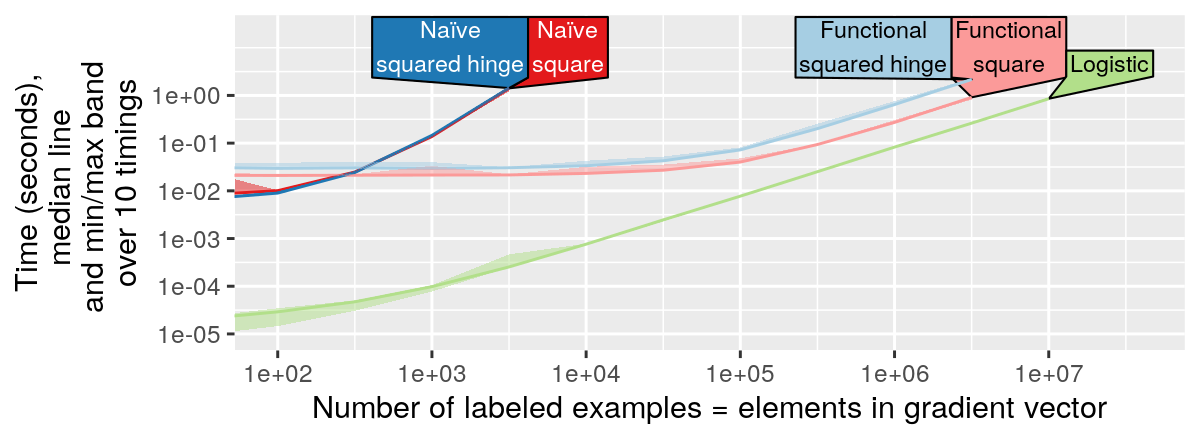}
    \vspace{-0.3cm}
    \caption{Proposed $O(n\log n)$ Functional algorithms for computing all pairs loss and gradient is asymptotically faster than compared with previous Na\"ive $O(n^2)$ algorithms, and almost as fast as $O(n)$ logistic loss. }
    \label{fig:timing-grad-square}
\end{figure}

\subsection{Experimental Details: data sets, splits, hyper-parameter grids}
In this section we will discuss how our algorithms compared to state-of-the-art methods on several machine learning benchmark data sets. The data sets utilized in these experiments consisted of CIFAR10, STL10, and Cat\&Dog \cite{Krizhevsky09learningmultiple} \cite{asirra-a-captcha-that-exploits-interest-aligned-manual-image-categorization}
\cite{coates2011stl10}.
These data sets were split first using their pre-defined standard train/test sets (each test set has no class imbalance, 50\% positive labels).
The CIFAR10 and Cat\&Dog data sets are typically utilized in multi-class problems, and therefore were converted into binary classification problems as described previously \cite{yuan2020auc}. 
Briefly, we set aside the first half of the class labels (ex. $y_i\:\in\:\{0,1,2,3,4\}$) as the negative class and the remaining labels ($y_i\:\in\:\{5,6,7,8,9\}$) were designated as the positive class.  
In order to achieve the desired train set class imbalance ratio (imratio = proportion of positive labels in train set = 0.1, 0.01, or 0.001), observations associated with positive examples were removed from the data set until the desired class imbalance was achieved.
From there the train set was split again into 80\% subtrain set (for computing gradients) and 20\% validation set (for hyper-parameter selection). 
In the following figures the experiments will be denoted as follows: ``train loss function + optimizer." 
Squared Hinge will refer to our new algorithm for computing the squared hinge loss. 
LIBAUC will refer to the AUCM loss \cite{yuan2020auc} and PESG will refer to the optimization algorithm proposed in \cite{guo2020pesg}. 
For comparison we also used the standard Logistic (binary cross entropy) loss function, with equal weights for each observation (this baseline is how most binary classifiers are trained without class imbalance / no special optimization for AUC).
In the following results, we report the mean/median across five different random seeds, each with a different random initialization of the neural network weights, and a different random subtrain/validation split.

In each of our experiments we used a pytorch implementation of the ResNet20 model; this is a deep neural network with 20 hidden layers, and a sigmoid last activation layer \citep{ResNet20}. 
We used the LIBAUC algorithm as a baseline \cite{yuan2020auc}, and that algorithm recommends using the sigmoid as the last activation, so we used that in our experiments as well in order to make a reasonable comparison.

In order to achieve maximum performance for each loss function and optimizer, each combination of train loss function + optimizer was trained over a grid of hyper-parameters, and the parameter combination and number of epochs that achieved the maximum validation AUC was selected.
The batch sizes were selected from $10,50,100,500,1000,5000$.
The learning rates tested were dependent on the loss function. For the LIBAUC and logistic loss functions the tested learning rates were $10^{-4},\dots,10^2$.
For the proposed square hinge loss the learning rates were tested across $10^{-4},\dots,10^{-1}$. 
We observed that if the learning rate is too large, when summing over all pairs of positive and negative examples, the resulting loss value very quickly diverges and will run into overflow problems. 
All batch sizes and learning rates were computed in parallel on a cluster in which each node had AMD EPYC 7542 CPUs.

\subsection{Selected hyper-parameter values when maximizing validation AUC}
In this section we wanted to investigate which hyper-parameters (batch size, learning rate) were selected by maximizing validation AUC, in order to investigate the extent to which large batch sizes are useful.
% It has been established in previous sections that the time complexity of our proposed method is greatly improved compared to the na\"ive square and squared hinge loss functions. 
Using the hyper-parameter grids described in the previous section, we wanted to examine if there are situations where it would be beneficial to select larger batch sizes that were not feasible using the quadratic time na\"ive methods. 
We expected that as the ratio of positive to negative labels decreased, it would be advantageous for the learning algorithm to examine more observations at once (so that each batch would have at least one example for each class). 
It can be observed from Table~\ref{tab:batch-size-lr-table} that  large batch sizes were selected for our proposed algorithm in several situations.
For example, using the Cat\&Dog data set and a proportion of positive examples of 0.001 in the train set, selecting the epoch with maximum validation AUC resulted in a median batch size of 1000 (over the five random initializations of the neural network weights). 
Attempting to use a batch size this large using the na\"ive method would have been very slow, given its time complexity of $O(n^2)$.
Overall these data provide convincing evidence that our proposed algorithm allows for efficient learning using large batch sizes, which are useful when the labels are highly imbalanced.
\begin{table}[t]
  \centering
  \resizebox{\textwidth}{!}{%
    \begin{tabular}{|c|l|r|r|r|r|r|r|}
    \hline
          &       & \multicolumn{2}{c|}{CIFAR10} & \multicolumn{2}{c|}{STL10} & \multicolumn{2}{c|}{Cat\&Dog} \\
    \hline
    \multicolumn{1}{|l|}{Imratio} & \multicolumn{1}{l|}{Loss Function}      & \multicolumn{1}{l|}{Batch} & \multicolumn{1}{l|}{Learning Rate} & \multicolumn{1}{l|}{Batch} & \multicolumn{1}{l|}{Learning Rate} & \multicolumn{1}{l|}{Batch} & \multicolumn{1}{l|}{Learning Rate} \\
    \hline
    \multirow{3}[2]{*}{0.1} & Our Square Hinge & 10    & 0.0316 & 10    & 0.0100 & 50    & 0.1000 \\
\cline{2-8}          & LIBAUC & 50    & 0.1000 & 50    & 0.1000 & 50    & 0.1000 \\
\cline{2-8}          & Logistic Loss & 10    & 0.1000 & 50    & 0.1000 & 50    & 1.0000 \\
    \hline
    \multirow{3}[2]{*}{0.01} & Our Square Hinge & 10    & 0.0032 & 100   & 0.1000 & 50    & 0.0316 \\
\cline{2-8}          & LIBAUC & 50    & 0.1000 & 1000  & 0.1000 & 100   & 0.1000 \\
\cline{2-8}          & Logistic Loss & 10    & 0.1000 & 1000  & 0.1000 & 100   & 1.0000 \\
    \hline
    \multirow{3}[2]{*}{0.001} & Our Square Hinge & \textbf{500}   & 0.0316 & 10    & 0.0001 & \textbf{1000}  & 0.3162 \\
\cline{2-8}          & LIBAUC & 100   & 10.0000 & 10    & 0.0001 & 500   & 10.0000 \\
\cline{2-8}          & Logistic Loss & 100   & 1.0000 & 100   & 0.0001 & 100   & 1.0000 \\
    \hline
    \end{tabular}%

    }
  \caption{Hyper-parameters which were selected by maximizing AUC on validation set (Median hyper-parameters over five random initializations of neural network). Batch column = batch size, Imratio column shows the proportion of positive examples in the train set. Bold for larger batch sizes which were selected using our proposed method.}
  \label{tab:batch-size-lr-table}
\end{table}

\subsection{Proposed Algorithm Has Similar or Better Test AUC In Cross Validation Experiments}

In this section we wanted to explore how much our proposed method  improves prediction accuracy on real data sets using a deep learning model (ResNet20). 
We expected that with the proposed loss function, we would be able to achieve similar or larger AUC values on the test set.

In Figure~\ref{fig:test-auc-at-max-valid-epoch}, it can be observed that our proposed algorithm performs as well as both baselines on data sets with 10\% positive examples in the train set (Imratio=0.1). One can see from Figure~\ref{fig:test-auc-at-max-valid-epoch} that the logistic loss performs well on data sets that are more balanced, but starts to fail once the ratio of positive to negative labels in the train set begins to decrease. 
Our proposed algorithm outperforms both baselines at the 0.01 class imbalance ratio. 
Using our proposed square hinge loss at this class imbalance consistently resulted in larger median test AUC values, in all three data sets.
At the class imbalance ratio of 0.001 all of the loss functions begin to struggle to detect the patterns in the data set (test AUC values only slightly larger than 0.5), but our proposed method does see some potential benefits when applied to the Cat\&Dog and CIFAR10 data sets.

\begin{figure}[t]
    \centering
    \includegraphics[width=1\textwidth]{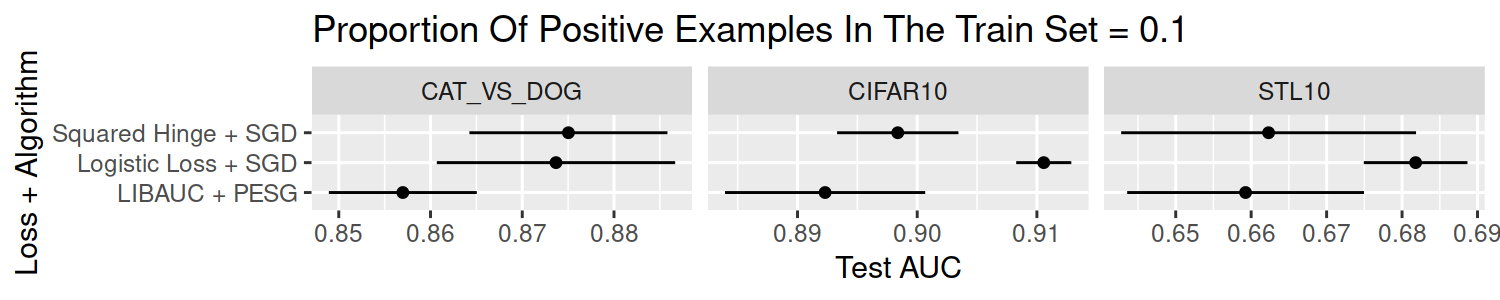}
    \includegraphics[width=1\textwidth]{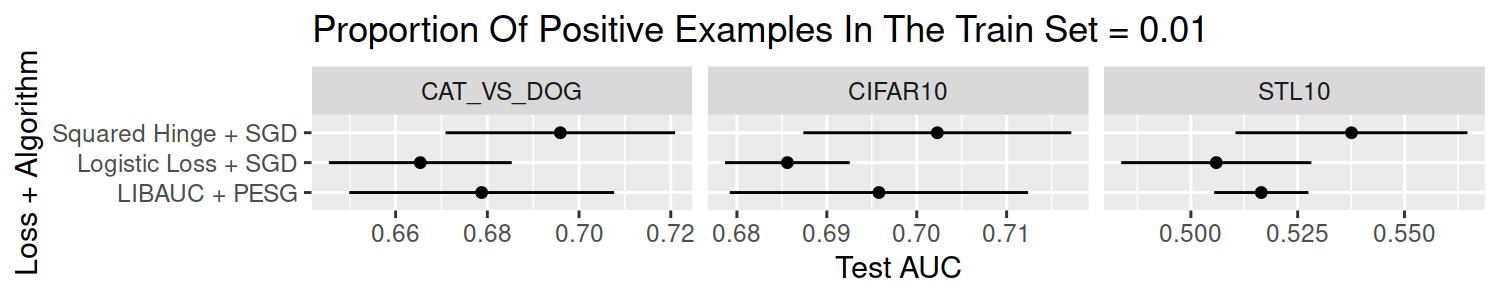}
    \includegraphics[width=1\textwidth]{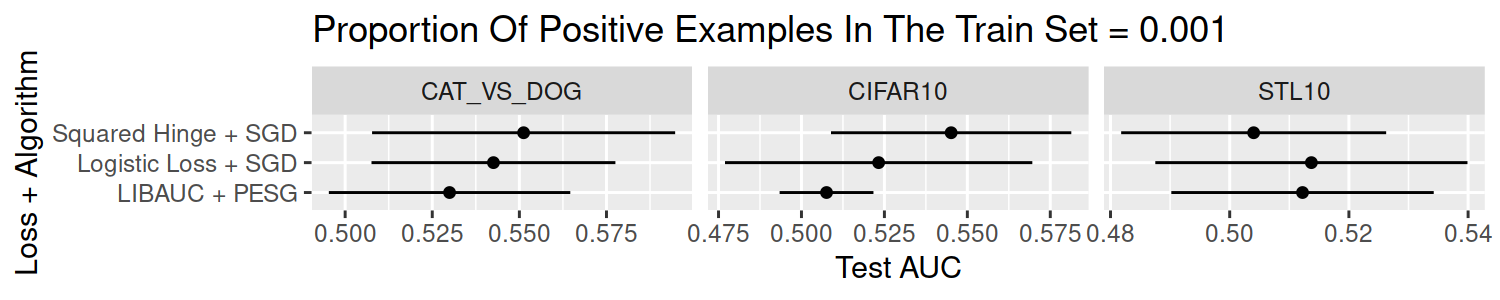}

    \caption{The AUC achieved on the test set (with no class imbalance) at the epoch that achieved the maximum AUC on the validation set (mean $\pm$ standard deviation over five random initializations of the neural network weights).}
    %png("figure.png", width=4, height=2, units="in", res=200)
    %print(gg)
    %dev.off()
    %imratio -> proportion of positive examples in train set
    \label{fig:test-auc-at-max-valid-epoch}
\end{figure}

\section{Discussion and conclusions}
\label{sec:discussion}
The main novel contribution of this paper is a log-linear time algorithm for computing the all pairs squared hinge loss, which has not been previously described.
Section~\ref{sec:algorithms} presented new formal proofs that a functional representation can be used to efficiently compute the all pairs square loss and squared hinge loss. 
We described new algorithms that can compute the square loss in linear $O(n)$ time, and the squared hinge loss in log-linear $O(n\log n)$ time (both would be quadratic time using a na\" ive implementation).

In our empirical timings experiment, we demonstrated how our proposed method compares to the na\"ive approach and the standard logistic loss. These algorithms have asymptotic time complexities of $O(n^2)$ and $O(n)$ respectively. 
It can be observed that our proposed algorithm's asymptotic timing vastly outperforms the na\"ive method, and is nearly as fast as the linear logistic loss. Our analysis shows that when a time limit of one second is enforced, the na\"ive method can complete the computation for data on the order of $n \approx 10^3$, while our proposed representation can compute $n \approx 10^6$ in the same amount of time.

In our second comparison experiment, we demonstrated using real-world data sets that there are cases when large batch sizes are selected by maximizing the AUC on the validation set. 
Our sub-quadratic algorithm is able to make use of larger batch sizes than previously feasible with an algorithm of quadratic time complexity. 
For a ratio of positive examples in the train set of 0.001, our method selected batch sizes of 500 and 1000 on the CIFAR10 and Cat\&Dog data sets respectively.

In our third comparison experiment, we compare the test AUC of our algorithm with both standard logistic loss and the state-of-the-art LIBAUC loss. 
It can be observed that our proposed method performs on par with these methods of low level of class imbalance. 
When the level of class imbalance is very high, none of the given loss functions do too well with capturing patterns in the data, but our method still allows the user to achieve a slightly higher AUC value. Our method clearly outperforms the other methods when the class imbalance is moderate across all three data sets (imratio=0.01 proportion of positive examples in train set), thus from our experiments this would be the most advantageous situation to utilize our proposed method.

A final potential advantage to our proposed method lies in interpretability. 
Our square hinge loss can be computed with respect to the entire subtrain/validation sets during each epoch of training, in the same $O(n\log n)$ time that it takes to compute AUC.
Our algorithm therefore makes it feasible to regularly monitor these quantities and use them to more easily diagnose problems in neural network training (for example, looking at subtrain loss to see if step size is too large/small).
% As described in \ref{sec:model}, our method's objective was to minimize the loss as opposed to directly maximizing AUC. 
% This allows the user to directly monitor the loss for the characteristic ``U" shape of overfitting on the validation set, which is not possible when using a method that directly maximizes the AUC.

In the future, we would like to investigate how our functional representation could be used when computing the linear hinge loss, which has non-differentiable points, so we could make use of sub-differential analysis. 
We would also like to explore for what problems it is advantageous to use a sigmoid versus linear last activation in neural networks with the loss functions we have proposed. 
We would like to explore how our method could be used with full batch sizes and deterministic optimization algorithms such as the Limited Memory Broyden–Fletcher–Goldfarb–Shanno (LBFGS) optimizer \citep{liu1989limited}. 
We expect that for problems where there exists a bad condition number, LBFGS with full batch size should out-perform Stochastic Gradient Descent with small batch sizes.
Finally, we would be interested to investigate how our method could be combined with the Stochastic Average Gradient algorithm in order to obtain accelerated convergence \citep{schmidt2017minimizing}.

\bibliographystyle{plainnat}
\bibliography{refs}

\newpage

\end{document}